\documentclass[10pt,twocolumn,letterpaper]{article}

\usepackage{wacv}              

\usepackage{booktabs}
\usepackage{siunitx}
\usepackage{multirow}
\usepackage{graphicx}
\usepackage{float}
\usepackage{caption}
\usepackage{stfloats}
\usepackage[utf8]{inputenc}
\usepackage{booktabs}  
\usepackage{array}     
\usepackage{adjustbox}
\usepackage{amsthm}
\usepackage{algpseudocode}
\usepackage{algorithm}
\usepackage{tabularx}
\usepackage{pdflscape} 
\usepackage{booktabs}
\usepackage[table]{xcolor} 
\usepackage{framed}
\definecolor{shadecolor}{gray}{0.92}

\definecolor{wacvblue}{rgb}{0.21,0.49,0.74}
\usepackage[pagebackref,breaklinks,colorlinks,allcolors=wacvblue]{hyperref}
\newtheorem*{problem}{Problem Statement}
\newtheorem{theorem}{Theorem}
\newtheorem{lemma}{Lemma}
\def\wacvPaperID{1703} 
\def\confName{WACV}
\def\confYear{2026}

\title{Human knowledge integrated multi-modal learning for single source domain generalization}

\author{
Ayan Banerjee \qquad Kuntal Thakur \qquad Sandeep Gupta\\
Impact Lab, Arizona State University\\
{\tt\small \{ayan.banerjee, kthakur9, sandeep.gupta\}@asu.edu}
}

\begin{document}
\maketitle
\begin{abstract}
Generalizing image classification across domains remains challenging in critical tasks such as fundus image–based diabetic retinopathy (DR) grading and resting-state fMRI seizure onset zone (SOZ) detection. When domains differ in unknown causal factors, achieving cross-domain generalization is difficult, and there is no established methodology to objectively assess such differences without direct metadata or protocol-level information from data collectors, which is typically inaccessible. We first introduce domain conformal bounds (DCB), a theoretical framework to evaluate whether domains diverge in unknown causal factors. Building on this, we propose GenEval, a multimodal Vision Language Models (VLM) approach that combines foundational models (e.g., MedGemma-4B) with human knowledge via Low-Rank Adaptation (LoRA) to bridge causal gaps and enhance single-source domain generalization (SDG). Across eight DR and two SOZ datasets, GenEval achieves superior SDG performance, with average accuracy of 69.2\% (DR) and 81\% (SOZ), outperforming the strongest baselines by 9.4\% and 1.8\%, respectively.
Code and models are available at: 
\url{https://github.com/IMPACTLabASU/GenEval}.
\end{abstract}
\begin{figure*}[t]
  \centering
  \includegraphics[width=0.95\textwidth]{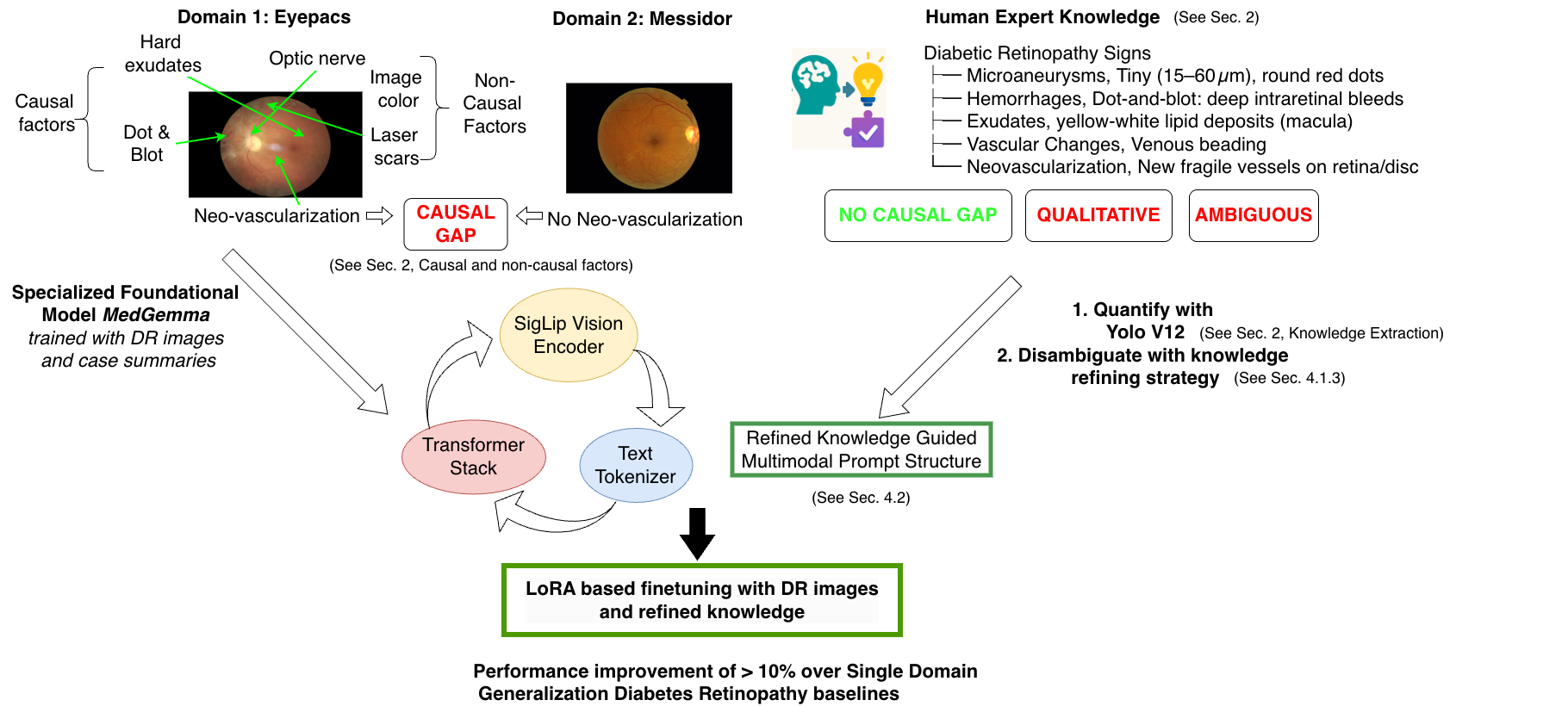}
  \caption{SDG attempts for DR has traditionally failed to consistently outperform ERM primarily due to existing gap in causal factors among data sets. Knowledge from human experts in the field can fill the causal gap but are usually qualitative and ambiguous. This paper demonstrates GenEval that quantifies human expert knowledge, refines them, and integrates specialized foundational model such as MedGemma-4B through LoRA based fine tuning to bridge the causal gap between domains while removing ambiguity of knowledge.}
  \label{fig:dr_stages}
\end{figure*}
\section{Introduction}
State-of-the-art (SOTA) supervised image classification engines face the problem of domain generalization (DG) which requires them to learn using image data $X_i \in X$ from source domains $\{D^s_i\}, i=\{1\ldots N\}$ collected in a specific environment, following a set protocol, using a given perception device, with a label set $Y$ and classify data points from target domains $\{D^t_i\}, i\in\{1\ldots M\}$ which may differ from $\{D^s_i\}$ in various aspects but with same class set. Traditionally, techniques to improve DG performance have not been successful in consistently outperforming empirical risk minimization (ERM)~\cite{arpit2022ensemble} or outperform ERM by statistically insignificant margins as seen for Diabetic Retinopathy (DR) grading~\cite{icdr_grading2003proposed} example in Table~\ref{tbl:MDG}. example in Table~\ref{tbl:MDG}. For example, when Messidor is the held-out target domain, SPSD-ViT achieves 64.0\% accuracy compared to ERM-ViT’s 62.7\%, yet the reported $p$-value of 0.09 indicates this improvement is not statistically significant. In practical clinical setting, access to data sources is limited that may prevent integration of datasets as required in multi-source DG (MDG). Hence a more practical problem faced in clinical setting is single source DG (SDG), where data from a single domain is available for training and the supervised classification technique should have good classification performance on target domains. This paper focuses on the more challenging setting of SDG, where SOTA techniques demonstrate even poorer performance than MDG~\cite{galappaththige2024spsd,gdrnet_miccai2023}.
  
\begin{table}
    \centering
     \caption{Comparison of SOTA DG improvement techniques across DR datasets. Leave-one-domain-out evaluation: each column shows accuracy when that domain is held out as the target and the remaining three are used for training. \textbf{\textit{Bold italic}} is highest accuracy, \underline{underlined} is second highest. p value denotes statistical significance of performance gain over ERM-ViT, p $>$ 0.05 implies insignificant result.}
    \scriptsize
    \begin{tabular}{@{}p{0.65 in} p{0.275 in} p{0.275 in} p{0.275 in} p{0.45 in} c@{}}
        \toprule
        \textbf{Method} & \textbf{Aptos} & \textbf{EyePACS} & \textbf{Messidor} & \textbf{Messidor 2} & \textbf{p value} \\
        \midrule
        MMD \cite{mmd2015long} & 49.3$\pm$1.0 & 69.3$\pm$1.1 & 64.1$\pm$4.8 & 69.6$\pm$0.6 & 0.3\\
        CDANN \cite{cdann2017long} & 48.1$\pm$0.7 & \underline{73.1$\pm$0.3} & 55.8$\pm$1.8 & 61.2$\pm$1.3 & 0.9\\
        SD-ViT \cite{sultana2022sd_vit} & 46.5$\pm$0.8 & 71.1$\pm$0.7 & 63.9$\pm$0.9 & 71.4$\pm$0.2 & 0.2\\
        SPSD-ViT \cite{galappaththige2024spsd} & \underline{51.6$\pm$1.1} & 73.3$\pm$0.4 & \underline{64.0$\pm$1.4} & \underline{72.9$\pm$0.1} & 0.09 \\
         ERM-ViT \cite{teterwak2025erm++} & \textit{48.5$\pm$0.9} & \textit{70.7$\pm$1.7} & \textit{62.7$\pm$1.6} & \textit{69.5$\pm$2.5} & NA \\
        GenEval (Ours) & \textbf{\textit{73.46}} & \textbf{\textit{83.18}} & \textbf{\textit{67.70}} & \textbf{\textit{79.64}} & \\
        \bottomrule
    \end{tabular}
    \label{tbl:MDG}
\vspace{0.1mm}
\end{table}
Recent theoretical advances~\cite{vuong2025domaingeneralizationfailview,cha2021swad,arpit2022ensemble} in generalization capacity of learned representations have uncovered two necessary conditions for feasibility of a machine $M$ that learns from a source domain $D^s$ to accurately classify data from target domain $D^t$: a) \textbf{causal cover}, source domain $D^s$ has representations of all causal factors that are required to classify data from target domain $D^t$, and b) \textbf{source risk minimization}, the machine $M$ is able to effectively learn all causal factors in the source domain.   

\noindent{\bf Causal Gap:} Domains of several critical medical applications including DR and resting state functional magnetic resonance imaging (rs-fMRI) based seizure onset zone (SOZ) detection may often violate causal cover condition. Figure \ref{fig:dr_stages} shows that DR images have both causal factors that directly affect DR grading classification such as hard exudates and Dot and Blot effects, and non-causal factors such as laser scars that do not directly affect classification but may have indirect correlation with DR severity since they are caused by frequent scans~\cite{Cuadros2009EyePACS}. Neo-vascularization is an important causal factor indicative of grade 4 DR (Proliferative DR). This is present in the EyePACS domain~\cite{Cuadros2009EyePACS}. However, no such indication is available in Messidor 1 domain~\cite{Decenciere2014Messidor}. This demonstrates a \textbf{causal gap} between the domains. Hence, a machine $M$ learning \textcolor{black}{classification logic} from Messidor 1 source is unlikely to perform accurate classification in the EyePACS domain~\cite{galappaththige2024spsd}. 

\noindent{\bf Solution hypothesis:} We observe that \textit{causal factors relevant for classification in any problem domain can be obtained from human expert knowledge} to compensate for causal gaps between domains. Hence a \textbf{multi-modal combination of image data and expert knowledge in textual form} can be used in a foundational vision language model to achieve better generalized performance across domains. 

\noindent{\bf Challenge:} Human knowledge is qualitative and ambiguous. For example, micro-aneurysms, a sign for Grade 1 DR are defined as tiny round red dots mostly between 15 - 60 $\mu m$ which can be easily confused with venous bleeding, which is a sign of Grade 3 DR. Effective usage of human knowledge requires accurate quantification and refinement to determine a subset of knowledge components that has maximum effect in reducing causal gap.  
Our principal {\bf contributions} are as follows (Figure \ref{fig:dr_stages}): 

\noindent{\bf 1. Novel causal gap quantification theory (Sec.~\ref{subsec:dcb}):} Based on the concept of conformal inference, we propose a novel theory of domain conformal bounds (DCB) that defines a distribution free mechanism to determine the difference in causal factor interrelationship between two domains. 

\noindent{\bf 2. Explanation of SOTA baseline DG performance (Sec.~\ref{subsec:sdcd}):} Based on the DCB theory, we define a novel metric, source domain conformance degree (SDCD) and prove that it has a positive correlation with SDG performance of a learning machine ($M$). This metric is used to explain the SDG performance of SOTA techniques.

\noindent{\bf 3. Knowledge quantification and refinement (Sec.~\ref{subsec:knowledge_refinement}):} We encode expert knowledge as real-valued vectors using a specialized single-stage detector and propositional logic, then use SDCD to measure causal-gap reduction and, via SDCD-guided ablations, select the subset that maximizes SDCD. 

\noindent{\bf 4. GenEval, multi-modal classification engine (Sec.~\ref{subsec:geneval_method}):} We fuse the refined knowledge and images into a multimodal prompt, fine-tune the specialized foundational model MedGemma-4B on source data via parameter-efficient Low-Rank Adaptation (LoRA)~\cite{hu2021lora}, then infer on the target domain.

\noindent\textbf{5. Comprehensive evaluation (Sec.~\ref{sec:evaluation})} across eight DR datasets (APTOS, EyePACS, Messidor, Messidor-2, IDRID, DeepDR, FGADR, RLDL) and two SOZ datasets (UNC, PCH), with emphasis on challenging SDG problem.

\section{Preliminaries}
\label{sec:prelim}
A domain $D$ is a tuple of image data $X$ with a probability distribution function $P(X)$. A source domain $D^s$ is a domain with a label set $Y$ for each image data in $X$. A target domain $D^t$ is a domain without any label set.
\begin{problem}{Single Source Domain Generalization -} Given a single source domain $D^s$, with labels $Y$ of $Q$ classes $\{C_1 \ldots C_Q\}$, learn a classification engine $\mathcal{M}$ to classify target domain $D^t$ data with the same $Q$ classes.    
\end{problem}

\noindent{\bf Causal and non-causal factors:} \textcolor{black}{Causal factors are application specific observations and correlations that are interpretable by experts and guide the labeling $Y$ for a dataset $X$.} Formally, the functional mapping between image data $X$ and label $Y$ in any domain is governed by causal $Z_c$ and non-causal $Z_n$ factors, such that data $X=F_x(Z_c,Z_n)$ is affected by both factors but the label is only affected by the causal factors $Y=F_y(Z_c)$, where $F_x$ and $F_y$ are nonlinear functions defining the image data $X$ and label $Y$.

\noindent{\bf Necessary conditions for DG:} Theoretical explorations in recent work have identified two necessary conditions for DG~\cite{arpit2022ensemble,vuong2025domaingeneralizationfailview}. A machine $M$ can learn from domain $D^s$ and accurately recognize target domain data $D^t$ if:  
\begin{enumerate}
    \item \textbf{Causal cover:} The target introduces no causal-factor \textcolor{black}{interrelationships} absent in the source domain.
    \item {\bf Source risk minimization:} the machine can effectively reduce cross entropy loss in the source domain.
\end{enumerate}

\noindent{\bf \textcolor{black}{Quantification of Causal Factor Interrelationship}:} \textcolor{black}{The source of causal factors is human expert knowledge. Given the causal factors, their interrelationship can be objectively evaluated from}: a) the data $X$ and label $Y$ through the function $F_x$ and $F_y$, or b) \textcolor{black}{expressing human expert knowledge obtained from various sources such as publications in the problem area or consultations as tacit logical propositions.} 

\noindent{\it \underline{a) Extraction of causal factor \textcolor{black}{interrelationship} from data}:} \textcolor{black}{Medical images capture physical processes with typically continuous spatio-temporal evolution modeled by partial differential equations (PDEs), constrained by boundary conditions. This assumption underlies image representations~\cite{Liu2014AdaptivePDE,Chen2015TNRD,Metzger2023DADA,Fainstein2024DUDF} and high-fidelity diffusion-based generation~\cite{song2021scorebasedgenerativemodelingstochastic}. PDEs with boundary conditions admit compartmental state-space realizations \cite{david2011hybrids,mathews2006nonlinear,slightly1998hybrid,bennett2005hybrid,camacho2007hybrid,smith2008hybrid}, allowing the data-generating maps \(F_x\) and \(F_y\) to be modeled as continuous multivariate nonlinear state spaces; consequently, Koopman theory \cite{koopman1931hamiltonian} applies, with boundaries/edges encoded via compartmental state variables.}

\noindent{\it Koopman theory:} Koopman theory linearizes nonlinear dynamics by applying a linear operator to observables in a lifted space, letting us expose causal interactions with linear tools. Koopman lifts the PDE-driven image dynamics to a sparse linear operator, letting us read causal factor interactions from nonlinear generators. We define a continuous relaxation $\Tilde{Y} = \Tilde{F}_y(Z_c)$ of the labeling function $F_y$ such that the original labels can be obtained from the relaxed labels $\Tilde{Y}$ by applying mutually exclusive thresholds for each class. The relaxed labels $\Tilde{Y}$ and image data $X$ can be expressed as a dynamical system defined by a function of causal factors $Z_c$, and non-causal factors $Z_n$ as -
\begin{equation}
    \scriptsize
    \Tilde{Y} = \Tilde{F}_y(Z_c), \frac{dX}{ds} = f_x(X,Z_c,Z_n).
\end{equation}
where $f_x=\frac{dF_x}{ds}$, and $s$ is an independent parameter used to traverse the image space.  

\noindent{\it Koopman approximations:} Koopman theory~\cite{chen2012variants} states that any nonlinear dynamical system can be expressed as a linear dynamical system in a transformed state space $\Omega$ through a measurement function $X = G(\Omega)$ such that: $\frac{d\Omega}{ds}=\mathcal{K}\Omega$, where $\mathcal{K}$ is the Koopman operator, usually a sparse coefficient matrix. Dynamic mode decomposition (DMD) and sparse identification of nonlinear dynamics (SINDY) limit $\Omega$ to a polynomial library over principal components of $X$; RIDGE regression~\cite{staal2004ridge} then estimates $\mathcal{K}$, with sequential thresholding forcing small coefficients to zero via $\mathcal{K}=STRIDGE(X,\Omega,G,\sigma)$~\cite{kaiser2018sparse,kaheman2020sindy}. Non-zero terms in the final $\mathcal{K}$ mark the causal factors in $X$.  

\noindent{\it Observation 1, difficulty in quantifying causal cover:} While source risk minimization can be quantified by monitoring the loss function, to the best of our knowledge there exists no such quantifier for causal cover. 

\noindent{\it Observation 2, image representation with differential dynamics:} \textcolor{black}{In practice, the functions \(f_x\) are (or can be closely approximated as) continuous and differentiable in imaging—especially medical imaging of spatially evolving physiology \cite{garnung2024physics,chan2003variational}. SOTA diffusion models likewise use differential dynamics to represent images \cite{song2021scorebasedgenerativemodelingstochastic}, and most prediction schemes trained with gradient descent implicitly assume differentiability \cite{arora2022understanding}. Even when \(f_x\) is non-differentiable or discontinuous, hybrid differential dynamics can represent it by augmenting the state to split dynamics across discontinuities, yielding piecewise models \cite{david2011hybrids,mathews2006nonlinear,slightly1998hybrid,bennett2005hybrid,camacho2007hybrid,smith2008hybrid}.}

\noindent{\it \underline{b) Extraction of causal factor \textcolor{black}{interrelationship} from human} \underline{knowledge:}} According to recent work~\cite{KambojTAI} human expert knowledge for a given class can be expressed in terms of propositional logic $\mathcal{L}_j(.)$. Knowledge comprises of semantic features $\mathcal{F}_i, i\in\{1\ldots R\}$ in the image data $X$. The presence of each semantic feature $\mathcal{F}_i$ can be quantified into real values. Human expert knowledge can be expressed as a first order propositional formula that expresses a class label $Y$ in terms of feature values in $\mathcal{F}$. Multiple such formulae define the oracle of knowledge on the classification problem. Truth evaluation of each formula provides a quantification of causal factor interrelationship embedded in human expert knowledge.  

\noindent{\bf Distance between two probability distribution of causal factors:} Given two multi-dimensional Gaussian distributions of $P=N(\mu_p,\Sigma_p)$ and $R=N(\mu_r,\Sigma_r)$, where $\mu_p,\mu_r$ are the means and $\Sigma_p,\Sigma_r$ are covariance matrices, the KL divergence~\cite{martins2004comparison} is a asymmetric metric that measures how different is distribution $R$ from $P$ as follows:
\begin{equation}
    \scriptsize
    d_k(P|R) = \frac{1}{2}[
    \ln\frac{\det \Sigma_r}{\det \Sigma_p}
    - k
    + \mathrm{tr}\bigl(\Sigma_r^{-1}\,\Sigma_p\bigr)
    + (\mu_r - \mu_p)^\top \,\Sigma_r^{-1}\,(\mu_r - \mu_p)],
\end{equation}
where $tr$ is trace and $det$ is determinant operation. It has been extensively used to explain domain shift~\cite{nguyen2022klguideddomainadaptation} and measure learning capability of AI engines~\cite{kim2021comparing}. 

\noindent{\it Observation 3, non-monotonicity:} The KL divergence is non-monotonous metric and cannot be used to quantitatively compare two distances: $d_k(R_1|R_2)$ and $d_k(R_1|P)$. It is thus not amenable to be used in optimization required for knowledge refinement.
In our paper, we utilize Mahalanobis distance $d_m$ given in Eqn. \ref{eqn:Maha}, to derive a monotonous divergence metric, \textit{source domain conformance degree} discussed in Section \ref{sec:Method}. Mahalanobis distance measures the distance of a point $x$ from a reference distribution as -
\begin{equation}
    \scriptsize
    \label{eqn:Maha}    
    d_m(x|P) = \sqrt{(x-\mu_p)^T\Sigma_p^{-1}(x-\mu_p)}
\end{equation}

\noindent{\it Observation 4, relation of Mahalanobis distance and KL divergence:} $d_m(x|P)$ appears as the last term of KL divergence, and has a positive correlation with KL divergence.
\paragraph{Knowledge Extraction}
 Domain-specific diagnostic rules were curated from human expert guidelines from ophthalmologists for DR (see Table~\ref{tab:dr_symptoms} in supplement) and epileptologists in SOZ (see Table \ref{tab:soz-clinical-signs} in supplement) and operationalized via automated feature extraction pipelines built using two open-source, modular tools: YOLOv12 and a retinal vessel segmentation model for DR and custom image processing code available in~\cite{Kamboj25CuKPL}. For lesion-level localization, we use YOLOv12, a SOTA one-stage detector for dense object environments that extends YOLOv5/YOLOv7 with C3K2 blocks, Spatial Pyramid Pooling Fast (SPPF), and C2PSA (Cross-Stage Partial Self-Attention)~\cite{tian2025yolov12attentioncentricrealtimeobject}. Its modified CSPDarkNet backbone and decoupled heads (objectness, classification, regression) achieve superior mean average precision (mAP) with real-time inference~\cite{wang2023comprehensive}. We train YOLOv12 to detect hemorrhages, hard exudates, and cotton wool spots; predicted bounding boxes are post-processed and validated via Intersection over Union (IoU) against expert-labeled fundus images, and each image yields up to a 14-dimensional real vector of IoU scores. For SOZ knowledge extraction we follow the image processing pipeline detailed in Kamboj et al~\cite{KambojTAI} and code shared in~\cite{Kamboj25CuKPL}.
\section{Related Work}
We situate GenEval-DR at the intersection of three threads:
\textit{(i)} medical domain generalisation (DG), and \textit{ii)} medical
vision–language models (VLMs), and \textit{(iii)} distribution-free
uncertainty quantification.

\noindent{\bf Domain Generalisation in Medical Imaging:} Performance of DR classifiers drops sharply when scanners, demographics
or prevalence shift across clinics~\cite{yoon2024domain_generalization_survey,dai2024deep}.
Classic DG methods align feature distributions with MMD or adversarial
losses~\cite{cdann2017long,mmd2015long,atwany2022drgen}, and extensive
augmentations such as BigAug simulate domain shift~\cite{zhang2020bigaug}.
Transformer baselines SD-ViT~\cite{sultana2022sd_vit} and the current
state-of-the-art SPSD-ViT~\cite{galappaththige2024spsd} replace CNN backbones
but still \textbf{assume target data are exchangeable with source
domains}. None of them can \emph{predict at deployment time} whether a
new domain data lies outside the training support, a prerequisite highlighted
by recent causal-DG theory~\cite{domain_survey_2024}.
GenEval-DR tackles this unmet need via an external DCB safety test.

\noindent{\bf Medical Vision Language Models:} Contrastive pre-training on paired image–text data enables zero-shot
transfer (BiomedCLIP~\cite{zhang2023biomedclip}, LLaVA-Med~\cite{li2023llava_med}).
For DR, CLIP-DR adapts CLIP with ranking-aware prompts and is the
strongest published VLM baseline~\cite{yu2024clipdr}.  Parameter-efficient
LoRA fine-tuning~\cite{hu2021lora} allows large medical VLMs such as
MedGemma-4B~\cite{medgemma_card_2025} to specialise to retinal imagery
without catastrophic forgetting. Yet VLMs alone remain brittle under
unseen shifts and lack uncertainty guarantees—an issue we address
by coupling MedGemma with DCB.

\noindent{\bf Causal structure change detection:} Changes in causal structure constitute anomalies. Although anomaly detection is well studied \cite{liu2024deep}, most methods are modality-specific and operate on raw data, whereas we require change detection over causal structures. Many techniques rely on extreme value theory (EVT) \cite{TSB} for out-of-distribution (OOD) detection, but EVT’s distributional assumptions cannot be verified \emph{a priori} for implicit causal structures, making it unreliable here. Conformal prediction provides finite-sample, distribution-free change detection \cite{vovk2005algorithmic,angelopoulos2021gentle,lu2021fair_conformal_medical,angelopoulos2024conformal_triage}.

\section{Methods}
\label{sec:Method}
We describe our theoretical framework (Figure. \ref{fig:geneval}) for causal cover evaluation and a novel VLM based technique for generalized image classification with DR as the example. 
\subsection{Evaluation of Causal Cover}
\label{subsec:eval_causal_cover}
\textcolor{black}{Given a source domain $D^s$ and target domain $D^t$, causal factors $\mathcal{K}$ can be extracted from each sample of each domain either from data or knowledge as shown in Section \ref{sec:prelim} with sets $D^s,D^t$ resulting in the sets $\mathcal{K}^s, \mathcal{K}^t$ respectively. We define a novel metric \textit{source domain conformance degree} (SDCD) denoted by $d_\rho$ that \textbf{evaluates the number of samples in the target domain $D^t$ that has the same causal factor interrelationship as the source domain $D^s$}. We use conformal inference~\cite{shafer2008tutorial,vovk2005algorithmic} to develop a distribution-agnostic algorithm that counts points in \(D^t\) whose causal factors fall within the distribution of the source-learned set \(\mathcal{K}^s\) from \(D^s\).} This is a two step process:
\subsubsection{Step 1: Domain conformal bounds (DCB)}
\label{subsec:dcb}
\textcolor{black}{Conceptually, DCB encodes the properties of the elements of a source domain $D^s$ into an interval that defines what it means for an image data to share the} \textcolor{black}{interrelationship patterns of causal factors} with image data in source domain. 
\textcolor{black}{We utilize Mahalanobis distance (Eqn. \ref{eqn:Maha}) to define a distance between two vectors of causal factor valuations derived from two image data points $X_i$ and $X_j$. For an image data $X_i$, the robustness metric $\rho(.,.)$ (Eqn. \ref{eqn:robust}) is the average Mahalanobis distance of $X_i$ from any other image $X^s$ in the source domain $D^s$. }
\begin{equation}
\scriptsize
\label{eqn:robust}
\rho(\mathcal{K}(X_i),D^s) = ({\sum\limits_{j=1,j\neq i}^{|D^s|}{d_m(\mathcal{K}(X_i),\mathcal{K}(X^s))}})/(|D^s|-1).
\end{equation} 
 Given the robustness function $\rho(.,.)$ in Eqn. \ref{eqn:robust}, conformal inference creates a prediction band $C \subset \mathcal{R}^2$ based on causal factor valuations from domain $D^s$ for a given \textit{miscoverage level} $\alpha \in \{0,1\}$, so that
\begin{equation}
\label{eqn:Property}
\scriptsize
Pr(\rho(\mathcal{K}(X_i),D^s) \in C) \geq 1 - \alpha, \textnormal{ iff } X_i \in D^s, \mathcal{K}(X_i) \in \mathcal{K}^s.
\end{equation} 
\textcolor{black}{We call this interval $C$, DCB and is a key property of source domain to encode interrelationships between causal factors obtained from any methodology. If the robustness metric of the causal factor valuations $\mathcal{K}$ from the data $X \in D^t$ in target domain falls inside the interval $C$ for source domain $D^s$, then image data $X\in D^t$ does not have any novel causal factor interrelationship that is not encompassed in images available in $D^s$ with probability $(1-\alpha)$. Algorithm \ref{alg:Prop} computes DCB $C$ defined by Eqn. \ref{eqn:Property} for a source domain $D^s$.}
\begin{algorithm}
	\caption{$C$ = \textbf{DCB Compute}($\{X_i Y_i\}_{i=1}^{|D^s|}$,$\alpha$,$\rho(.,.)$,$\mathcal{K}^s$)}
	\scriptsize
\begin{algorithmic}[1]
		\State Split $D^s$ into two equal sized subsets $I_T$ and $I_V$.
		\State Average robustness $\sigma = avg_i(\rho(\mathcal{K}(X_i),I_T{/\{X_i,Y_i\}}))$
        \State For each $X_v,Y_v$ compute residual $R_j = \rho(\mathcal{K}(X_v),I_T) - \sigma$
		\State $d$ = the kth smallest value in $\{R_j:j\in I_V\}$, $k = \lceil(|I_V|/2+1)(1-\alpha)\rceil$ 
        \State {\bf return} $C = [d-std(R_j),d+std(R_j)]$
\end{algorithmic}
	\label{alg:Prop}
\end{algorithm}
 Algorithm \ref{alg:Prop} takes the training data $(X_i,Y_i)\in D^s$ of source domain, miscoverage level $\alpha$ and the set of all causal factor model estimations $\mathcal{K}^s$ as input and provides DCB as output. It divides the training set, data and their model estimates, into two mutually exclusive subsets $I_T$ and $I_V$. For each $\mathcal{K}(X_i) \in I_T$, $\rho(\mathcal{K}(X_i),I_T{/\{X_i,Y_i\}})$ is computed, where $I_T{/\{X_i,Y_i\}}$ denotes the set $I_T$ with $\{X_i,Y_i\}$ removed. Let $\sigma = avg_i(\rho(\mathcal{K}(X_i),I_T{/\{X_i,Y_i\}}))$ be the mean value of the robustness metric in the training set $I_T$. From the validation set $I_V$ residual $R_v\rho(\mathcal{K}(X_v),I_T) - \sigma$ is derived for every element in $I_V$, the residual is arranged in ascending order. The algorithm then finds the residual $\sigma_C$ at the position $\lceil(|I_V|/2+1)(1-\alpha)\rceil$ and is used as the prediction range $C=[\sigma_C - std(R_v),\sigma_C+std(R_v)]$. We prove Theorem \ref{th:ThreePOne} that guarantees that the interval $C$ obtain from Algorithm \ref{alg:Prop} satisfies the guarantee in Eqn. \ref{eqn:Property}.
\begin{theorem}
\label{th:ThreePOne}
For any data point $X,Y \in D^s$, $Pr(\rho(\mathcal{K}(X),D^s) \in C = [\sigma_C- d, \sigma_C+d] ) \geq 1 - \alpha$, $\alpha > 0$ if and only if $Pr(\mathcal{K}(X)|\{X\}\in D^s) = Pr(\mathcal{K}(X)|\{X\}\in D^t)$, where $d$ is given by Algorithm \ref{alg:Prop}. (Proof in supplement) 
\end{theorem}  
\begin{figure*}[t]
  \centering
  \includegraphics[width=0.9\textwidth]{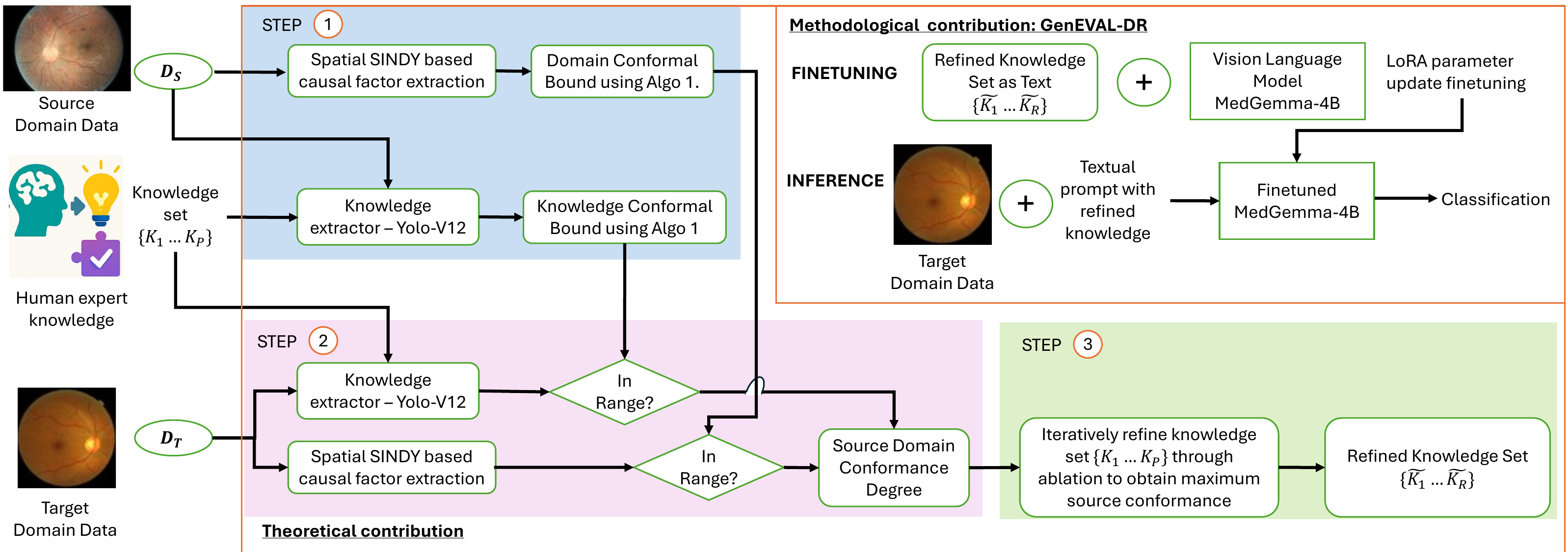}
  \caption{Contributions: a) Algorithms \ref{alg:Prop} and \ref{alg:U2} to obtain domain conformal bounds and source domain conformance degree. Utilizing SDCD, ablation studies are performed to refine the knowledge set into an optimal set that maximizes SDCD. b) GenEval, a specialized foundational vision language model based method that integrates image and expert knowledge in textual form for a multi-modal prompt driven DR grading and SOZ detection schemes.}
  \label{fig:geneval}
\end{figure*}
 \begin{algorithm}[H]
	\caption{\textbf{SDCD}($D^s$,$D^t$,$\rho(.,.)$,$\mathcal{K}^s$,$\mathcal{K}^t$,$\sigma$,$C$)}
	\scriptsize
\begin{algorithmic}[1]
		\State \textbf{output} Percentage of $D^t$ with causal cover.
		\For {each $\{X\} \in D^t$ }
		\State Compute residual $R = \rho(\mathcal{K}(X),D^s) - \sigma$
		\If{$R \in C$} 
            \State add to the list of satisfaction
            \EndIf
            \EndFor
		\State\textbf{return} percentage causal cover satisfaction  
\end{algorithmic}
	\label{alg:U2}
\end{algorithm}
 \subsubsection{Step 2: Source domain conformance degree}
\label{subsec:sdcd}
 Utilizing Theorem \ref{th:ThreePOne} and Algorithm \ref{alg:Prop}, we derived a robustness range that encodes the causal factors of source domain $D^s$. The \textit{SDCD} extraction mechanism in Algorithm \ref{alg:U2} takes each sample from target domain $D^t$ and without its labels, evaluates whether the sample is within the DCB of $D^s$.  
 \begin{lemma}
 \label{lem:SDCD}
 The SDCD metric is: i) a monotonous function of the robustness residue in Eqn. \ref{eqn:robust}, and ii) has positive correlation with the performance of a learning machine $M$ that minimizes empirical risk in source domain $D^s$ and tested on target domain for large datasets. 
\noindent{\bf Proof:} Discussed in supplement Section \ref{sec:Proof}.
 \end{lemma}
\subsubsection{Step 3: Knowledge refinement through ablation}
\label{subsec:knowledge_refinement}
Using SDCD metric we can carry out ablation studies to determine the best combination of knowledge that has the maximum SDCD. This will guarantee the best accuracy for any machine $M$ that integrates the refined knowledge according to Lemma \ref{lem:SDCD}. We use the quantification strategy described in Section \ref{sec:prelim} to convert the expert knowledge component to a 14 dimensional real vector. Through several iterations we remove the knowledge components one by one thereby reducing the dimension of the real vector and compute the average SDCD metric for each pair of source and target domain. We keep the configuration that gives the best average SDCD metric value.  
\subsection{GenEval Method for SDG in DR}
\label{subsec:geneval_method}
\textbf{Medical Vision-Language Foundation.} GenEval leverages MedGemma-4B~\cite{medgemma_card_2025} VLM pre-trained on extensive medical image-text pairs. Unlike general-purpose VLMs that struggle with clinical diagnostic reasoning~\cite{radford2021clip,zhang2023biomedclip,li2023llava_med}, MedGemma-4B incorporates medical knowledge through its training on clinical reports, textbooks, and peer-reviewed literature. The model architecture consists of three key components working in synergy~\cite{anatomical_vlm_miccai2024,pathchat2024medical,shen2023med_flamingo}: a) The clinical vision encoder processes $224 \times 224$ pixel images through a multi-head attention mechanism specifically optimized for medically relevant components such as retinal structures~\cite{gulshan2016development,dai2024deep}. It captures both low-level pathological features (microaneurysms, hemorrhages, exudates) and high-level diagnostic patterns (neovascularization, macular edema, venous beading)~\cite{icdr_grading2003proposed,zhang2023biomedclip,anatomical_vlm_miccai2024,li2023llava_med}. 
\begin{table}[t]
  \centering
  \caption{LoRA setup and fine-tuning time. A = \textsc{APTOS 2019}, E = \textsc{EyePACS}, M1 = \textsc{Messidor-1}, M2 = \textsc{Messidor-2}.}
  \label{tab:lora-params}
  \scriptsize
  \setlength{\tabcolsep}{6pt}
  \begin{tabular}{@{} p{3.2in} @{}} 
    \toprule
    \textbf{Key LoRA hyperparameters} \\
    Adapter rank ($r$) = 16; scaling factor ($\alpha$) = 16; LoRA dropout = 0.05; 
    target modules = \{\texttt{q\_proj, k\_proj, v\_proj, o\_proj, up\_proj, down\_proj, gate\_proj}\}; 
    modules to save = \{\texttt{lm\_head, embed\_tokens}\}; 
    trainable params = $\approx$95M (2.4\% of 4B); task type = \texttt{CAUSAL\_LM}. \\[0.25em]
    \hline
    \textbf{Training time (hh:mm:ss)} \\
    \textit{Single domain:} A = 01:03:35; E = 10:14:20; M1 = 00:20:41; M2 = 00:34:11. \\
    \textit{Multi domain:} A+E+M1 = 12:01:30; A+E+M2 = 12:23:19; A+M1+M2 = 01:58:43; E+M1+M2 = 11:18:07. \\
    \hline
    \textbf{ Inference time}\\
    \textit{Single domain:} 423.96 ± 33.52 ms [383.14, 458.72]; \textit{Multi domain:} 385.46 ± 29.24 ms [359.71, 427.35].\\
    \hline
    \textbf{YOLO lesion detector} \textit{End-to-end (wall):} 632.78 $\pm$ 119.58 ms [528.58, 783.62]. \\
    \bottomrule
  \end{tabular}
\end{table}
\textbf{Parameter-Efficient Clinical Adaptation.} To preserve MedGemma-4B's extensive pre-trained clinical knowledge while adapting to specific classification tasks, we employ Low-Rank Adaptation (LoRA)~\cite{hu2021lora} with novel architectural modifications inspired by recent advances in asymmetric parameter-efficient fine-tuning~\cite{liu2024hydralora,adamw2017decoupled}. For each transformer layer $l$, we inject trainable low-rank matrices following the decomposition:
\begin{equation}
\scriptsize
\mathbf{W}_l' = \mathbf{W}_l + \alpha \cdot \mathbf{B}_l\mathbf{A}_l
\end{equation}
where $\mathbf{W}_l \in \mathbb{R}^{d \times d}$ represents the frozen pre-trained weight matrix, $\mathbf{B}_l \in \mathbb{R}^{d \times r}$ and $\mathbf{A}_l \in \mathbb{R}^{r \times d}$ are trainable matrices with rank $r \ll d$, and $\alpha$ is a scaling factor~\cite{hu2021lora,liu2024hydralora}. We set $r = 16$ and $\alpha = 16$ based on empirical validation across multiple DR datasets, following established optimization protocols~\cite{cosine_annealing2016sgdr,gradient_clipping2013difficulty}.
This approach provides three critical advantages for medical applications~\cite{anatomical_vlm_miccai2024,pathchat2024medical}. \textbf{Computational efficiency} is achieved by updating only approximately 2\% of total parameters ($\sim$95M out of 4B), drastically reducing GPU memory requirements and training time while maintaining competitive performance~\cite{hu2021lora,liu2024hydralora}. \textbf{Knowledge preservation} ensures that the frozen weights retain broad clinical reasoning capabilities acquired during pre-training, preventing catastrophic forgetting of medical knowledge~\cite{zhang2023biomedclip,li2023llava_med}. \textbf{Clinical interpretability} emerges from the low-rank structure, as the adapter matrices $\mathbf{B}_l$ and $\mathbf{A}_l$ can be analyzed to understand task-specific diagnostic adjustments and feature importance patterns~\cite{mutual_information_neural2018mine,anatomical_vlm_miccai2024}.
\textbf{Knowledge refinement and structured clinical prompting.} Effective DR classification requires prompts that mirror clinical diagnostic workflows and leverage the International Clinical Diabetic Retinopathy Disease Severity Scale~\cite{icdr_grading2003proposed,label_smoothing2016rethinking}. We develop two complementary prompting strategies optimized for different deployment scenarios, drawing from established clinical practice guidelines and recent advances in medical language model prompting~\cite{pathchat2024medical,shen2023med_flamingo,anatomical_vlm_miccai2024}. 
\noindent{\bf Zero shot evaluation:} We employ a comprehensive clinical instruction that positions MedGemma-4B as a specialist diagnostic system: 
\begin{snugshade}
\scriptsize
\texttt{``You are a medical AI system specialized in diabetic retinopathy 
diagnosis from retinal fundus images. Classify DR severity (0--4) based on 
established clinical criteria: 0—No DR: normal retina with no 
microaneurysms; 1—Mild NPDR: few microaneurysms only; 2—Moderate NPDR: multiple 
microaneurysms, hemorrhages, possibly cotton-wool spots; 3—Severe NPDR: 
extensive hemorrhages in all four quadrants, venous beading, IRMA without 
neovascularization; 4—Proliferative DR: neovascularization, vitreous 
hemorrhage, or retinal detachment. Respond with only the number (0--4).''}
\end{snugshade}
\noindent\textbf{Fine-tuned models}, we incorporate role-playing elements and contextual medical reasoning through an expert consultation prompt~\cite{zhang2023biomedclip,li2023llava_med}:
\begin{snugshade}
\scriptsize
\texttt{``You are an ophthalmologist with 15 years of experience in 
diabetic retinopathy screening. Examine the macula, optic disc, and 
four quadrants for microaneurysms, hemorrhages, hard/soft exudates, 
venous abnormalities, and IRMA. Classify using \{0,1,2,3,4\} and 
return only the severity number.''}
\end{snugshade}

\noindent The refined knowledge obtained from the SDCD-based refinement strategy is incorporated directly into the prompt. This structured approach incorporates clinical reasoning patterns and systematic examination protocols that mirror actual diagnostic workflows~\cite{gulshan2016development,dai2024deep,icdr_grading2003proposed}, improving both accuracy and consistency in model predictions across different clinical contexts.

\section{Evaluation, Results and Analysis}
\label{sec:evaluation}
There are five evaluation goals: a) Validating the monotonicity property in Lemma \ref{lem:SDCD}, b) quantify the effect of knowledge integration on SDCD; c) MDG performance of GenEval, d) SDG performance of GenEval, and e) Performance comparison of GenEval with untuned VLMs.
\noindent{\bf Benchmark datasets:} DR datasets: EyePACS~\cite{Cuadros2009EyePACS}, APTOS~\cite{APTOS2019BlindnessDetection}, Messidor-1~\cite{Decenciere2014Messidor}, Messidor-2~\cite{Messidor2Consortium2024}. Labels are harmonized to the 5-class ICDR scale (0 to 4) using dataset grades or standard mappings. In the \emph{extended SDG} setting, we fix EyePACS as the source and evaluate on four additional DR targets: IDRiD~\cite{h25w98-18}, DeepDR~\cite{LIU2022100512}, FGADR~\cite{9257400}, and RLDL~\cite{wei2021learn}. Seizure onset zone (SOZ): two institutional rs-fMRI cohorts, Phoenix Children’s Hospital (PCH) and University of North Carolina (UNC), processed with one unified pipeline; SOZ labels use a 3-class scheme 0–2 where 0 = noise, 1 = resting state, and 2 = SOZ~\cite{Kamboj25CuKPL}.
\noindent{\bf Baselines:} For DR, we use the strongest baselines reported in~\cite{galappaththige2024spsd}, DECO~\cite{deco_miccai2024} and GDRNet~\cite{gdrnet_miccai2023}. For untuned VLMs on DR, we include CLIP~\cite{radford2021clip}, OrdinalCLIP~\cite{li2022ordinalclip}, and CLIP-DR~\cite{yu2024clipdr}. For SOZ cross-site SDG, we compare against the CuPKL for zero-shot rare-event classification~\cite{Kamboj25CuKPL}.

\noindent{\bf Metrics:} For DR (ICDR 0–4), our primary metric for SDG/MDG is overall multi-class accuracy on each target domain. For DR comparisons with untuned VLMs, we report macro F1 (mean F1 over the five DR classes). For SOZ cross-site SDG (4-class: 0–3), we use macro F1 as the primary metric due to class imbalance, and additionally report accuracy, macro precision, and macro recall in the supplement. We also report SDCD (\%) for causal-cover analysis (computed with $\alpha{=}0.05$ DCB) where relevant.
\noindent{\bf Statistics:} We use the t-test to evaluate p value of accuracy difference between GenEval and baselines. SDG tables are repeated with p values in appendix Section 
\ref{sec:rep}

\begin{table}[htbp]
\centering
\scriptsize
\caption{SDG results on diabetic retinopathy datasets. Knowledge is \textbf{K} and Data is \textbf{D} We evaluate our approach across different source-target domain pairs using state-of-the-art methods. All results are reported as classification accuracy (\%).}
\label{tab:domain_generalization_results}
\begin{tabular}{@{}p{0.25 in}@{}p{0.3 in}p{0.625 in}@{}p{0.425 in}@{}p{0.35 in}@{}p{0.25 in}p{0.25 in}p{0.25 in}@{}}
\toprule
\textbf{Sou rce} & \textbf{Target} & \textbf{Best Baseline} & \textbf{Baseline Accuracy} & \textbf{K SDCD} & \textbf{D SDCD} & \textbf{GenEval} & \textbf{K + D SDCD} \\
\midrule
{Mess} & Aptos & SPSD-ViT~\cite{galappaththige2024spsd} & 48.3 \%& 92.7\% & 16\% & 56.0\% & 98.03\% \\
{idor} & Messidor2 & SPSD-ViT~\cite{galappaththige2024spsd} & 62.2\% & 99.1\% & 87.1\% & 65.18\% & 99.14\% \\
& EyePACS & SPSD-ViT~\cite{galappaththige2024spsd} & 57.4\% & 17.02\% & 36.4\% & 80.04\% & 94.94\% \\
\midrule
{Mess} & Aptos & SPSD-ViT~\cite{galappaththige2024spsd} & 52.8\% & 32.1 \%& 17.7\% & 69.69\% & 76.3\%\\
idor2 & Messidor & SD-ViT~\cite{sultana2022sd_vit} & 61.0\% & 99.9\% & 98.2\% & 67.67\% & 100.00\% \\
& EyePACS & SPSD-ViT~\cite{galappaththige2024spsd} & 72.5\% & 34.6\% & 41.31\% & 77.79\% & 96.34\% \\
\midrule
\multirow{3}{*}{Aptos} & Messidor2 & DRGen~\cite{atwany2022drgen} & 61.0\% & 62.5\% & 79.04\% & 64.01\% & 88.10\% \\
& Messidor & DRGen~\cite{atwany2022drgen} & 46.7\% & 56.4 \%& 73.9\% & 49.0\% & 100.00\% \\
& EyePACS & SD-ViT~\cite{sultana2022sd_vit} & 72.0\% & 59.4\% & 77.82\% & 77.78\% & 99.97\% \\
\midrule
{Eye} & Aptos & SPSD-ViT~\cite{galappaththige2024spsd} & 75.1\% & 46.9\% & 99.8\% & 73.16\% & 99.84\% \\
pacs & Messidor2 & DRGen~\cite{atwany2022drgen} & 65.4\% & 99.2\% & 99.7\% & 80.5\% & 99.83\% \\
& Messidor & DRGen~\cite{atwany2022drgen} & 54.6\% & 99.1\% & 71.1\% & 69.48\% & 100.00\% \\
\bottomrule
\end{tabular}

\end{table}
\vspace{-0.175 in}\paragraph{1. Validation of Lemma \ref{lem:SDCD}}
The SDCD from data is shown in Table \ref{tab:domain_generalization_results}, column \textbf{D SDCD}. The Pearson's correlation coefficient of 0.692 (p 0.02) showing a positive correlation between SDCD and Accuracy of the best baseline. This demonstrates the monotonicity of SDCD in real world data. 
\begin{table}[h]
\centering
\scriptsize
\caption{Knowledge refinement through ablations guided by SDCD and corresponding accuracy.}
\label{tab:ablation_sdcd_accuracy}
\begin{tabular}{lcc}
\toprule
\textbf{Ablation} & \textbf{SDCD (\%)} & \textbf{Accuracy (\%)} \\
no ablation & 59.000 & 65.01\\
remove microaneurysms & 68.000 & 70.02 \\
remove hemorrhages, or exudates  & 71.683 & 71.05 \\
remove venous beading & 82.765 & 73.21 \\
\textbf{remove Neovascularization}  & \textbf{82.81} & \textbf{73.23} \\
\bottomrule
\end{tabular}
\vspace{0.1mm}
\end{table}

\paragraph{2. Effect of knowledge integration}
Table \ref{tab:domain_generalization_results} shows the SDCD extracted from refined knowledge (column \textbf{K SDCD}) and the SDCD extracted from the integration of knowledge and data (column \textbf{K + D SDCD}). It shows significant improvement in SDCD with respect to data. The integration also helped GenEval to increase accuracy in all SDG settings. 
\noindent{\it Knowledge refinement using SDCD:} Table \ref{tab:ablation_sdcd_accuracy} shows the SDCD for removing different knowledge components. We see that removing neovascularization gives the best SDCD. This is because neovascularization or new bleeding veins are extremely difficult to identify and YOLOv12 fails to do so. We also see that SDCD is correlated with accuracy and removing neovascularization gives the highest accuracy. This results in the prompt  
for fine tuned models.

\paragraph{3. Sensitivity analysis of SDCD }
\begin{figure}[t]
  \centering  \includegraphics[width=0.8\columnwidth]{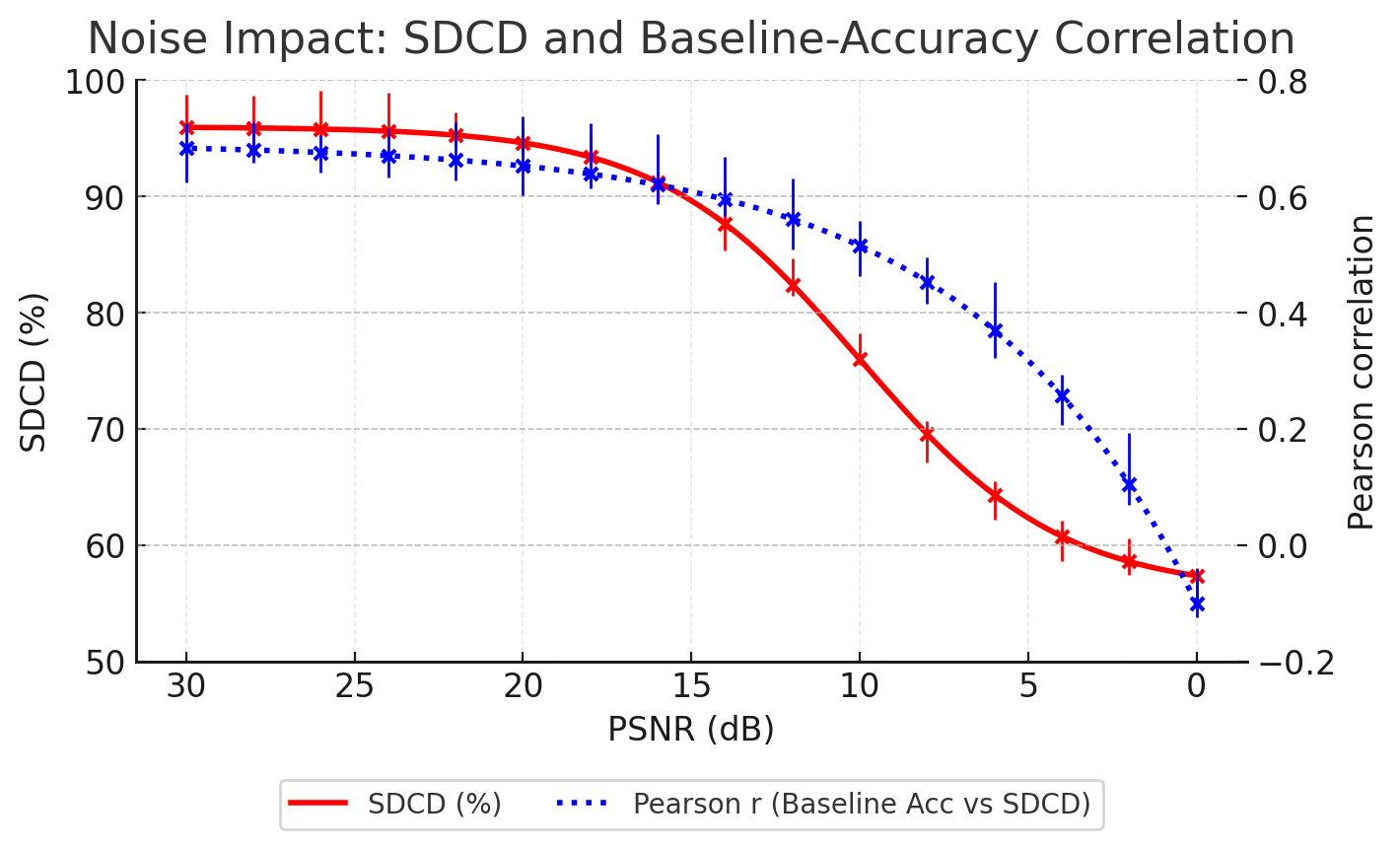}
  \caption{Variation of SDCD metric with pixel signal to noise ratio due to instability in Mahalanobis distance, and effect on correlation of SDCD with Accuracy given by Lemma \ref{lem:SDCD}.}
  \label{fig:SDCDFig}
\vspace{0.1mm}
\end{figure}
The Mahalanobis distance scales inversely with the covariance matrix; thus, as noise increases it becomes unstable, affecting SDCD. Using Table~\ref{tab:domain_generalization_results} with EyePACS as source and all other datasets as targets, we added Gaussian noise and swept peak signal to noise ratio (PSNR) from 30\,dB to 0\,dB. As PSNR decreases, Figure~\ref{fig:SDCDFig} shows SDCD (from raw data) stays high initially, then drops rapidly to an average of $\approx56\%$ across datasets. Concurrently, the Pearson correlation between SDCD and classification accuracy falls from $0.692$ ($p<0.02$) to $-0.1$ ($p<0.8$). SDCD also depends on YOLO-based knowledge extraction for DR: YOLOv11 outperforms YOLOv12, yielding slightly higher SDCD (Table~\ref{tab:Yolo} in supplement). However, when combined with data, the final $\text{K}+\text{D}$ SDCD increases only marginally, producing a statistically insignificant gain in GenEval accuracy.
\vspace{-0.175 in}\paragraph{4. Multi-Source Domain Generalization}
Table~\ref{tbl:MDG} presents fair comparisons where all methods are trained on three datasets and evaluated on the fourth. GenEval demonstrates substantial (+5.71\%) improvement over SPSD-ViT (79.21\% vs 73.3\%),
the best performing SOTA.
\vspace{-0.175 in}\paragraph{5. Single-Source Domain Generalization}
Table \ref{tab:domain_generalization_results} summarizes the SDG results for GenEval as compared to the best baseline available for the given datasets. Tables \ref{tab:eyepacs_source_pvalue} through \ref{tab:messidor2_source_pvalue} in supplement present systematic evaluation across all 12 source-target pairs and more baselines.
\noindent\textit{Cross-Domain Transfer Analysis:} Clinical datasets with standardized protocols (EyePACS, Messidor series) serve as more effective source domains than resource-limited collections (APTOS). The performance hierarchy reflects imaging quality and acquisition standardization, with implications for real-world deployment strategies.
\noindent\textit{Domain Difficulty Assessment:} APTOS consistently presents the greatest challenge both as source and target domain, reflecting the reality of deploying AI systems in resource-constrained environments. 
\noindent{\it Extended Single-Source Domain Generalization:} we fix the training domain to EyePACS and vary many targets. This is more challenging that varying domains and evaluating on a fixed large-scale target set since some source domain is bound to have good agreement with the target. Our design matches the deployment we care about in practice (train once on a single source clinic, deploy widely) and avoids conflating results with the choice of source. Concretely, we add IDRiD \cite{h25w98-18}, DeepDR \cite{LIU2022100512}, FGADR \cite{9257400}, and RLDL \cite{wei2021learn} to the core targets so that the only changing factor is the target-domain shift, not the training model or pipeline. For comparability, we use published single-source accuracies from prior work for baselines rather than re-implementing them, because their setups change the training domain while ours holds it fixed. This keeps the comparison apples-to-apples for \emph{fixed-source, multi-target} SDG and makes clear that any gains reflect robustness to diverse target distributions under a single, harmonized source.
\begin{table}[h]
  \centering
  \caption{Extended single-source domain generalization on diabetic retinopathy. Train on EyePACS, test on six external targets. Metric: accuracy (\%). Best scores in \textbf{bold}.}
  \label{tab:dr-extended-sdg}
  \scriptsize
  \setlength{\tabcolsep}{4pt}
  \renewcommand{\arraystretch}{1.15}
  \begin{tabular}{@{} l r r r r r r r @{}} 
    \toprule
    \textbf{Method} & \textbf{APTOS} & \textbf{Messidor} & \textbf{IDRiD} & \textbf{DeepDR} & \textbf{FGADR} & \textbf{RLDL} & \textbf{Avg} \\
    \hline
    GDRNet & 52.8 & 65.7 & 70.0 & 40.0 & 7.5 & 44.3 & 46.7 \\
    DECO   & 59.7 & \textbf{70.1} & \textbf{74.8} & 40.3 & 9.9 & 49.3 & 50.68 \\
    GenEval & \textbf{73.2} & 69.5 & 70.6 & \textbf{59.2} & \textbf{56.9} & \textbf{67.6} & \textbf{66.2} \\
    \bottomrule
  \end{tabular}
\end{table}
GenEval outperforms DECO \cite{deco_miccai2024} and GDRNet\cite{gdrnet_miccai2023} on all benchmarks for the challenging task of SDG on five additional datasets without any change to the model trained with EyePACS dataset. Largest gains occur on DeepDR, FGADR, and RLDL while performance is near-parity on Messidor (-0.6 pp) and slightly lower on IDRiD (-4.2 pp). 
Under a fixed-source, multi-target design, these results benchmarks a deployment-relevant scenario applicable broadly across clinics.
\begin{table}[!htbp]
\centering
\caption{VLM SDG F1 score (\%) for DR Classification.}
\label{tab:vlm_performance}
\small
\setlength{\tabcolsep}{10pt}
\renewcommand{\arraystretch}{1.2}
\begin{tabular}{lccc}
\toprule
\textbf{VLM Method} & \textbf{APTOS} & \textbf{Messidor} & \textbf{Average} \\
\midrule
CLIP~\cite{radford2021clip}          & 44.3 & 39.6 & 41.9 \\
OrdinalCLIP~\cite{li2022ordinalclip} & 45.7 & 41.8 & 43.8 \\
CLIP-DR~\cite{yu2024clipdr}          & 46.3 & 47.3 & 46.8 \\
GenEval- (Ours)          & \textbf{72.0} & \textbf{78.2} & \textbf{75.1} \\
\bottomrule
\end{tabular}
\end{table}
\vspace{-0.175 in}\paragraph{6. Vision-Language Model Performance}
GenEval-VLM demonstrates substantial improvements over existing vision-language approaches for cross-domain diabetic retinopathy classification, achieving significant F1 score gains of +28.3\% average improvement across challenging domain shifts \cite{radford2021learning,li2022ordinalclip,yu2024clipdr}. Our method achieves 72.0\% F1 score on APTOS and 78.2\% on Messidor datasets, representing improvements of +25.7\% and +30.9\% respectively over the previous best-performing CLIP-DR method, while maintaining competitive AUC performance \cite{anatomical_vlm_miccai2024,zhang2023biomedclip}. The integration of domain conformal boundaries with MedGemma-4B's medical pre-training enables robust generalization across diverse clinical environments, with 95\% coverage guarantee providing essential reliability assessment for clinical deployment \cite{angelopoulos2024conformal_triage,mehrtens2024conformal_pitfalls,hirsch2024noise_robust_conformal}. These performance improvements, combined with parameter-efficient LoRA adaptation, position GenEval-VLM as a clinically viable solution for multi-institutional DR screening programs where both accuracy and reliability assessment are critical for patient safety \cite{trustworthy_clinical_ai_2024,hu2021lora}.
\vspace{-0.175 in}\paragraph{7. Single source domain generalization on SOZ }
\label{sec:soz-sdg}
Table~\ref{tab:soz-sdg-f1} shows that GenEval achieves superior average F1 performance (90.0\%) compared to CuPKL GPT-4o (88.1\%), demonstrating more consistent generalization across centers. Notably, while CuPKL excels on PCH (93.8\% F1), it shows reduced performance on UNC (82.3\% F1), whereas GenEval maintains more stable performance across both sites (89.0\% and 91.0\% F1, respectively).

\begin{table}[H]
  \centering
  \caption{Single-source domain generalization on SOZ detection. Results are reported as F1-score (\%). Best scores in \textbf{bold}.}
  \label{tab:soz-sdg-f1}
  \scriptsize
  \setlength{\tabcolsep}{8pt}
  \renewcommand{\arraystretch}{1.15}
  \begin{tabular}{@{} p{0.1 in} p{0.1 in} r p{0.5 in} p{0.4 in} p{0.6 in} @{}} 
    \toprule
    \textbf{Source} & \textbf{Target} & \textbf{CuPKL \cite{Kamboj25CuKPL}} & CuPKL Avg & \textbf{GenEval} & {GenEval Avg} \\
    \hline
    PCH & UNC & 82.3 & 88.1 &\textbf{91.0} & 90 \\
    UNC & PCH & \textbf{93.8} & & 89.0 & \\
    \bottomrule
  \end{tabular}
\end{table}

The superior average F1 score indicates GenEval's enhanced ability to balance precision and recall in detecting epileptogenic regions across different clinical environments. This consistency is particularly valuable for clinical deployment where reliable performance across diverse institutional settings is essential. We report detailed accuracy, precision, and recall metrics in supplementary Table~\ref{tab:soz-sdg}, along with comprehensive cross-site transfer analysis.

\section{Conclusions}
\label{sec:conclusion}
SDG is a cornerstone challenge for deploying AI systems in real‐world settings, arguably the Achilles’ heel of machine learning, and especially critical in medical imaging where shifts between clinical sites or imaging devices can severely degrade performance. To probe the feasibility of transferring a model to an unseen target domain, our paper rigorously evaluates two complementary aspects: first, whether the source domain furnishes the necessary causal cover for the target distribution, and second, integrated human expert knowledge with foundational models to potentially encapsulate all underlying causal factors that drive variation across domains. We shown that a knowledge refinement mechanism is necessary to derive an optimal set of knowledge component that reduces ambiguity. GenEval-DR then integrates DR images and refined knowledge as textual prompts to classify DR images. GenEval-DR not only predicts generalization success, but also empowers researchers to make principled choices about which source domains to include during training –paving the way for a formal, theory-based treatment of DG that relies on necessary and sufficient conditions rather than heuristic fixes.

\textbf{Limitations:} GenEval-DR assumes that the underlying data‐generation mechanism is continuous and differentiable. In real-world systems, particularly those that combine digital computations with physical measurement devices, sharp transitions or threshold effects can violate this assumption. Although such non-differentiable behavior can be locally approximated by switched hybrid differentiable systems around discontinuities, these approximations may introduce errors in causal factor estimation and model evaluation. Developing methods that explicitly account for or adapt to genuine discontinuities, and quantifying the impact of approximation error on generalization metrics, remains an important direction for future research.
\enlargethispage{2\baselineskip} 
\vspace{-1mm}
{\samepage
\section*{Acknowledgments}
\noindent This work is partly funded by DARPA AMP-N6600120C4020, DARPA FIRE-P000050426, NSF FDT-Biotech 2436801, and the Helmsley Charitable Trust (2-SRA-2017-503-M-B). We thank Farhat Shaikh, Midhat Urooj, and other Impact Lab collaborators for valuable discussions during problem formulation, and Harwinder Singh, Vedanti Dantwale, and Vaishnavi Gadhikar for their assistance with experiments and validation.
}
{
    \small
    \bibliographystyle{ieeenat_fullname}
    \bibliography{main,sec/neurips}
}

\twocolumn
\section{Supplementary Material}
\subsection{Causal Factor Relationship Extraction via PySINDy}
To capture the underlying causal structure of domain-specific variations in diabetic retinopathy datasets, we employ Sparse Identification of Nonlinear Dynamics (SINDy)~\cite{brunton2016discovering,kaiser2018sparse} to extract interpretable causal factors from fundus images. Our approach transforms static retinal images into dynamical representations that reveal intrinsic patterns governing disease manifestation across different clinical environments.
\subsubsection{Theoretical Framework}
For each dataset $\mathcal{D}_i$, we conceptualize fundus images as observations of an underlying dynamical system whose evolution captures both pathological progression and domain-specific characteristics. Unlike traditional feature extraction methods that treat images as static entities, our approach models the spatial variations across the fundus as trajectories in a high-dimensional state space, where radial profiles from the optic disc to the periphery encode critical diagnostic information~\cite{gulshan2016development,dai2024deep}.
The SINDy framework identifies sparse dynamical relationships by constructing a library of candidate functions:
\begin{equation}
\Theta(\mathbf{z}) = [\mathbf{z}, \mathbf{z}^2, \mathbf{z}^3, \ldots, \mathbf{z}^5, \sin(\mathbf{z}), \cos(\mathbf{z}), \ldots]
\end{equation}
and solving the sparse regression problem:
\begin{equation}
\dot{\mathbf{z}} = \Theta(\mathbf{z})\boldsymbol{\xi}
\end{equation}
where $\boldsymbol{\xi}$ represents the sparse coefficient vector encoding causal relationships. The sparsity constraint, enforced through Sequential Threshold Least Squares (STLSQ) optimization, ensures that only the most significant dynamical modes are retained, yielding interpretable representations~\cite{brunton2016discovering,champion2019unified}.
\subsubsection{Causal Factor Representation}
This process generates domain-specific theta coefficients $\boldsymbol{\theta}_i = \{\theta_{ijk}\}$ for each dataset, where $\theta_{ijk}$ quantifies the strength of nonlinear interaction between spatial features $j$ and $k$ in generating feature $i$. These coefficients capture:
\begin{itemize}
    \item \textbf{Pathological dynamics}: How disease markers (microaneurysms, hemorrhages, exudates) evolve spatially across the fundus
    \item \textbf{Domain invariants}: Fundamental patterns consistent across imaging protocols
    \item \textbf{Domain-specific modulations}: Variations induced by equipment, illumination, and acquisition protocols
\end{itemize}
The resulting theta matrices form a compact yet expressive representation of the causal factors affecting data distribution across different clinical environments~\cite{kaiser2018sparse,kaptanoglu2022pysindy}.
\subsubsection{Statistical Characterization via Mahalanobis Distance}
Given the extracted causal factor representations $\boldsymbol{\theta}$ from each domain, we flatten these multi-dimensional structures into vectors $\mathbf{t} \in \mathbb{R}^{C \times P}$ where $C$ represents the number of identified causal interaction types and $P$ the polynomial degree (5 in our experiments). This flattening preserves the relational structure while enabling statistical analysis.
For the source domain $\mathcal{D}_S$, we establish its characteristic causal distribution through:
\begin{equation}
\boldsymbol{\Sigma}_S = \frac{1}{N_S-1}\sum_{i=1}^{N_S}(\mathbf{t}_i^{(S)} - \boldsymbol{\mu}_S)(\mathbf{t}_i^{(S)} - \boldsymbol{\mu}_S)^T
\end{equation}
where $\boldsymbol{\mu}_S = \frac{1}{N_S}\sum_{i=1}^{N_S}\mathbf{t}_i^{(S)}$ represents the mean causal factor configuration.
The Mahalanobis distance provides a natural metric for quantifying deviations from this characteristic distribution:
\begin{equation}
\rho(\mathbf{t}, \mathcal{D}_S) = \sqrt{(\mathbf{t} - \boldsymbol{\mu}_S)^T\boldsymbol{\Sigma}_S^{-1}(\mathbf{t} - \boldsymbol{\mu}_S)}
\end{equation}
This metric accounts for the covariance structure of causal factors, providing a more nuanced measure than Euclidean distance. The inverse covariance matrix $\boldsymbol{\Sigma}_S^{-1}$ acts as a precision matrix, emphasizing deviations along directions of low variance that may indicate significant domain shift~\cite{de2000mahalanobis,leys2018detecting}.
\begin{figure*}[t]
  \centering
  \includegraphics[width=0.80\textwidth,
    trim=30 60 30 48,clip]{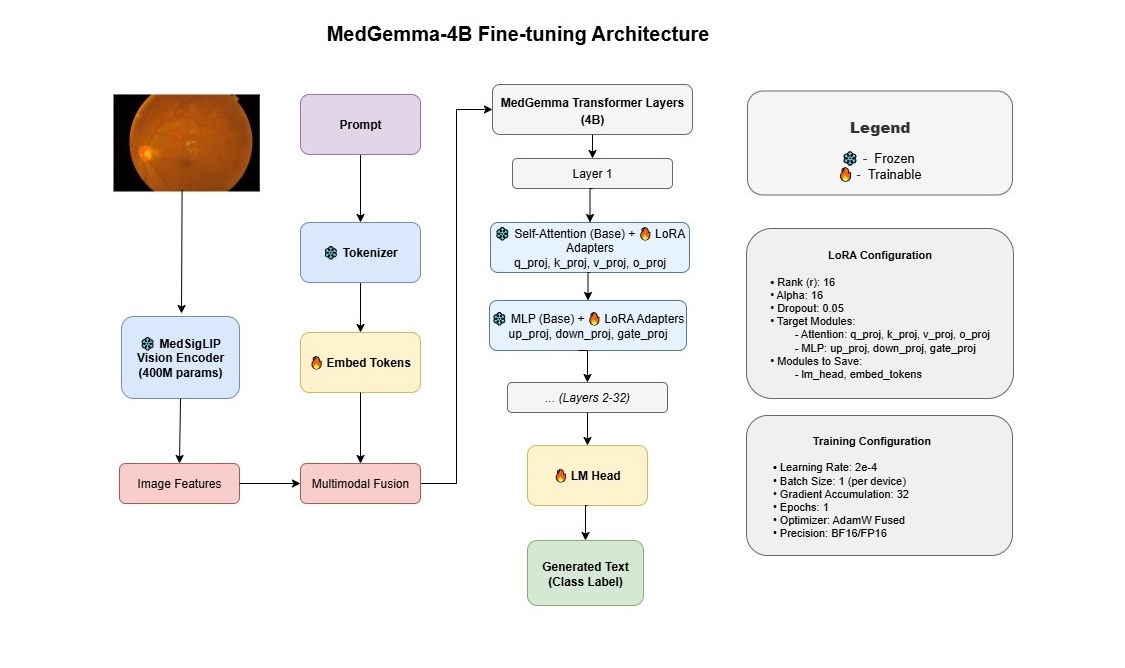}
  \caption{MedGemma-4B LoRA fine-tuning overview.\protect\footnotemark}
  \label{fig:supp-finetune-architecture}
\end{figure*}
\footnotetext{Rank 16, $\alpha=16$, dropout 0.05, targets: q,k,v,o, up,down,gate; LR $2\times10^{-4}$; batch 1 (accum 32); 1 epoch; AdamW fused; bf16/fp16.}
\subsection{Proofs}
\label{sec:Proof}
\subsubsection{Domain Conformal Boundaries}
\noindent{\bf Proof of Theorem \ref{th:ThreePOne}:} To establish rigorous statistical guarantees on domain compatibility, we employ conformal prediction methodology adapted for causal factor distributions. This approach provides distribution-free finite-sample guarantees without assuming specific parametric forms~\cite{vovk2005algorithmic,shafer2008tutorial}.
For target domain samples with causal factors $\{\mathbf{t}_j^{(T)}\}_{j=1}^{N_T}$, we compute adjusted robustness scores that normalize for the source domain's intrinsic variability:
\begin{equation}
\rho_{\text{adj}}(\mathbf{t}_j^{(T)}) = \rho(\mathbf{t}_j^{(T)}, \mathcal{D}_S) - \bar{\rho}_S
\end{equation}
where $\bar{\rho}_S$ represents the mean intra-domain distance, serving as a baseline for expected variations.
The conformal boundary is established through quantile-based calibration:
\begin{equation}
\tau_\alpha = \text{Quantile}_{\lceil (N_T/2 + 1)(1-\alpha) \rceil}\left(\{\rho_{\text{adj}}(\mathbf{t}_j^{(T)})\}_{j=1}^{N_T}\right)
\end{equation}
This boundary provides the critical threshold: with probability at least $1-\alpha$, samples from domains with compatible causal structure will satisfy $\rho_{\text{adj}}(\mathbf{t}) \leq \tau_\alpha$. For our experiments, we set $\alpha = 0.05$, yielding 95\% confidence bounds~\cite{angelopoulos2021gentle,tibshirani2019conformal}.
\subsubsection{SDCD metric computation}
\begin{lemma}
 \label{lem:SDCDV2}
 The SDCD metric is: i) a monotonous function of the robustness residue in Eqn. \ref{eqn:robust}, and ii) has positive correlation with the performance of a learning machine $M$ that minimizes empirical risk in source domain $D^s$ and tested on target domain for large datasets. 
 \noindent{\bf Proof sketch part i:} Since the interval $C$ is closed and complete, if an image has high robustness residue then the probability of the residue falling outside $C$ monotonously increases resulting in decrease of SDCD.
\noindent{\bf Proof sketch part ii:} ERM in $D^s$ guarantees that the machine $M$ has internal representation $\mathcal{G}(K^s)$ as a function of causal factors $K^s$ that can be derived from any data in $D^s$. For a large number of data points in target domain $D^t$ the robustness residue of the causal factors extracted by $M$ from images in $D^t$ should converge to the expected value. Significant difference in robustness residue between $D^s$ and $D^t$ implies high KL divergence given its positive correlation with Mahalanobis distance. Higher KL divergence implies higher cross entropy loss when tested on $D^t$ resulting in poorer accuracy~\cite{kim2021comparing,nguyen2022klguideddomainadaptation}. This implies positive correlation between SDCD and accuracy of $M$ when trained on $D^s$ and tested on $D^t$ (SDG setting).   
 \end{lemma}
\subsection{Importance of Fine-Tuning: Comparative Analysis}
\label{subsec:importance_finetuning}
Before evaluating domain generalization (DG) strategies, we quantify how far a strong medical vision–language prior (MedGemma-4B) can go \emph{without} any diabetic retinopathy (DR)–specific adaptation. The zero shot results in Table~\ref{tab:zero_shot} show that pre-training alone produces moderate performance (average accuracy \(71.73\%\), F\(_1\) \(70.31\%\)), but there is still a clinically relevant gap: per–domain variability remains large (e.g., only accuracy \(61.80\%\) in APTOS vs. \(79.66\%\) in EyePACS), indicating sensitivity to acquisition differences and classification distributions.
When we introduce supervised fine-tuning within the multi-source DG setting (see Table~\ref{tbl:MDG}), performance on the held-out domains improves relative to corresponding zero-shot figures, and several DG methods fail to surpass (or even match) a properly fine-tuned empirical risk minimization (ERM) baseline. This underscores two key points:
(i) zero-shot deployment of a large medical VLM is insufficient for robust cross-domain DR grading; and
(ii) careful fine-tuning already recovers a substantial portion of the attainable accuracy, establishing a strong baseline against which DG enhancements must be judged. Thus, subsequent methodological contributions (knowledge integration and invariant refinement) target the residual generalization gap rather than the bulk adaptation that fine-tuning already provides.
\begin{table}[t]
  \centering
  \caption{Zero-shot performance of MedGemma-4B (no DR-specific fine-tuning).}
  \label{tab:zero_shot}
  \small
  \setlength{\tabcolsep}{9pt}
  \renewcommand{\arraystretch}{1.2}
  \begin{tabular}{lccc}
    \toprule
    \textbf{Dataset} & \textbf{Accuracy (\%)} & \textbf{F1-Score (\%)} & \textbf{Recall (\%)} \\
    \midrule
    APTOS       & 61.80 & 60.25 & 62.15 \\
    EyePACS     & 79.66 & 77.61 & 78.95 \\
    Messidor    & 67.92 & 67.58 & 69.44 \\
    Messidor-2  & 77.52 & 75.78 & 76.71 \\
    \midrule
    \textbf{Average} & \textbf{71.73} & \textbf{70.31} & \textbf{71.31} \\
    \bottomrule
  \end{tabular}
\end{table}
\subsection{Clinical Relevance of DR Symptoms}
Table~\ref{tab:dr_symptoms} provides detailed clinical context for the diabetic retinopathy symptoms that our model learns to recognize through fine-tuning.
\begin{table}[h]
  \centering
  \caption{Clinical Signs of DR and Their Diagnostic Significance}
  \label{tab:dr_symptoms}
  \footnotesize
  \begin{tabularx}{\columnwidth}{@{}lX@{}}
    \toprule
    \textbf{Symptom} & \textbf{Key Observations and Diagnostic Relevance} \\
    \midrule
    Microaneurysms & Tiny red capillary dilations in the retina; earliest sign of mild NPDR. Their progression correlates with disease severity~\cite{bhavsar2019diabetic}. \\
    \addlinespace
    Hemorrhages & Includes dark red dot/blot and flame-shaped types indicating microvascular leakage. Severe NPDR is marked by more than 20 hemorrhages in all quadrants; risk of PDR rises to 50\% within a year~\cite{wong2018guidelines}. \\
    \addlinespace
    Hard Exudates & Sharp yellow lipid-rich deposits from chronic leakage, often in/near the macula. Indicative of risk for diabetic macular edema (DME), a major cause of vision loss~\cite{etdrs1991grading}. \\
    \addlinespace
    Cotton Wool Spots & Fluffy white retinal lesions caused by nerve-fiber-layer infarctions. Signify retinal ischemia in moderate to severe NPDR~\cite{diabetic_retinopathy_statpearls}. \\
    \addlinespace
    Subhyaloid Hemorrhages & Boat- or D-shaped hemorrhages between the retina and the hyaloid face, typically from ruptured neovascular vessels. Hallmark of proliferative DR~\cite{campbell1995subhyaloid}. \\
    \addlinespace
    Neovascularization & Fragile new vessel growth on the optic disc (NVD) or elsewhere on the retina (NVE). Defining trait of PDR; untreated cases face ~60\% vision loss within five years~\cite{dandona2001pdr}. \\
    \bottomrule
  \end{tabularx}
\end{table}
\subsection{IoU Scores for Lesion Localization}
To quantify spatial alignment between model detections and expert lesion annotations, we compute Intersection-over-Union (IoU) for three clinically salient lesion categories (Microaneurysms, Hemorrhages, Hard Exudates). For each category we rasterize (i) the union of predicted bounding boxes / masks and (ii) the union of ground-truth annotations into binary masks $M_{\text{pred}}$ and $M_{\text{gt}}$ and use
\[
\mathrm{IoU} = \frac{|M_{\text{pred}} \cap M_{\text{gt}}|}{|M_{\text{pred}} \cup M_{\text{gt}}|}.
\]
The \emph{Composite} IoU is the pixel–area–weighted IoU over the three lesion classes (not a plain arithmetic mean), giving larger lesions proportionally more influence.
\begin{table}[t]
  \centering
  \scriptsize
  \caption{Sample per–dataset IoU (\%) $\pm$ std.\ dev.\ on a held-out validation subset (illustrative). Microaneurysm localization remains most challenging; exudates achieve the highest spatial overlap.}
  \label{tab:sample_iou}
  \begin{tabular}{@{}lcccc@{}}
    \toprule
    \textbf{Dataset} & \textbf{Micro.} & \textbf{Hemorr.} & \textbf{Exud.} & \textbf{Composite} \\
    \midrule
    APTOS      & 42.3 $\pm$ 5.1 & 48.6 $\pm$ 4.7 & 61.2 $\pm$ 5.4 & 51.0 $\pm$ 4.2 \\
    EyePACS    & 46.1 $\pm$ 4.3 & 54.0 $\pm$ 4.9 & 65.8 $\pm$ 5.0 & 55.2 $\pm$ 4.1 \\
    Messidor   & 44.7 $\pm$ 5.6 & 50.2 $\pm$ 5.1 & 63.1 $\pm$ 5.5 & 53.0 $\pm$ 4.6 \\
    Messidor-2 & 49.0 $\pm$ 4.8 & 56.3 $\pm$ 4.5 & 68.4 $\pm$ 4.7 & 57.0 $\pm$ 3.9 \\
    \midrule
    \textbf{Mean} & \textbf{45.5} & \textbf{52.3} & \textbf{64.6} & \textbf{54.1} \\
    \bottomrule
  \end{tabular}
\end{table}
\noindent \textbf{Interpretation.} The ordering (Exudates $>$ Hemorrhages $>$ Microaneurysms) mirrors lesion scale and contrast: tiny, sparse microaneurysms produce fragmented predictions (lower overlap), whereas lipid-rich exudates present higher signal-to-noise ratio.
\subsection{Results Table with p values}
\label{sec:rep}
\begin{table}[!htbp]
  \centering
  \caption{SDG: Training on EyePACS. Accuracy \%}
  \label{tab:eyepacs_source_pvalue}
  \scriptsize
  \setlength{\tabcolsep}{4pt}
  \begin{adjustbox}{max width=\linewidth}
    \begin{tabular}{lcccc}
      \hline
      \textbf{Method}                  & \textbf{APTOS}      & \textbf{Messidor}   & \textbf{Messidor-2} & \textbf{Average} \\
      \hline
      DRGen \cite{atwany2022drgen}     & 61.3$\pm$1.9        & \underline{54.6$\pm$1.5}        & \underline{65.4$\pm$0.1}        & 60.4             \\
      ERM-ViT \cite{teterwak2025erm++} & 69.1$\pm$1.4       & 50.4$\pm$0.3        & 62.8$\pm$0.2        & 60.8             \\
      SD-ViT \cite{sultana2022sd_vit}  & 69.3$\pm$0.3        & 50.0$\pm$0.5        & 62.9$\pm$0.2       & 60.7             \\
      SPSD-ViT \cite{galappaththige2024spsd}
                                        & \textbf{75.1$\pm$0.5} & 50.5$\pm$0.8        & 62.2$\pm$0.4        & \underline{62.5}             \\
      GenEval (Ours)                   & \underline{73.2$\pm$ 0.4}                & \textbf{69.5$\pm$ 0.8}       & \textbf{80.5$\pm$ 0.4}       & \textbf{74.4}    \\\hline
      p value &0.5 & $<$0.01 & $<$0.01 & $<$0.01\\
      \hline
    \end{tabular}
  \end{adjustbox}
\end{table}
\begin{table}[!htbp]
\centering
\caption{SDG: Training on Aptos. Accuracy \%}
\label{tab:aptos_source_pvalue}
\scriptsize
  \setlength{\tabcolsep}{4pt}
\begin{tabular}{lcccc}
\hline
\textbf{Method}                  & \textbf{EyePACS}  & \textbf{Messidor} & \textbf{Messidor-2} & \textbf{Average} \\
\hline
DRGen \cite{atwany2022drgen}     & 67.5$\pm$1.8      & \underline{46.7$\pm$0.1}      & \underline{61.0$\pm$0.1}     & 58.4             \\
ERM-ViT \cite{teterwak2025erm++} & 67.8$\pm$1.4      & 45.5$\pm$0.2      & 58.8$\pm$0.4       & 57.3             \\
SD-ViT \cite{sultana2022sd_vit}  & \underline{72.0$\pm$0.8}      & 45.4$\pm$0.1      & 58.5$\pm$0.2       & \underline{58.6}             \\
SPSD-ViT \cite{galappaththige2024spsd} & 71.4$\pm$0.8 & 45.6$\pm$0.1      & 58.8$\pm$0.2       & \underline{58.6}           \\
{GenEval (Ours)}          & \textbf{77.8$\pm$0.8}     & \textbf{49.0$\pm$0.2}     & \textbf{64.0$\pm$0.2}      & \textbf{63.6}    \\
\hline
p value & $<$0.01& 0.04& 0.02 & $<$0.01\\\hline
\end{tabular}
\end{table}
\begin{table}[!htbp]
\centering
\caption{SDG: Training on Messidor. Accuracy \%}
\label{tab:messidor1_source_pvalue}
\scriptsize
  \setlength{\tabcolsep}{4pt}
\begin{tabular}{lcccc}
\hline
\textbf{Method}                  & \textbf{APTOS}    & \textbf{EyePACS}  & \textbf{Messidor-2} & \textbf{Average} \\
\hline
DRGen \cite{atwany2022drgen}     & 41.7$\pm$4.3      & 43.1$\pm$7.9      & 44.8$\pm$0.9       & 43.2             \\
ERM-ViT \cite{teterwak2025erm++} & 45.3$\pm$1.3     & 52.4$\pm$3.2      & 58.2$\pm$3.2      & 51.9             \\
SD-ViT \cite{sultana2022sd_vit}  & 44.3$\pm$0.9      & 53.2$\pm$1.6      & 57.8$\pm$2.4       & 51.7             \\
SPSD-ViT \cite{galappaththige2024spsd} & \underline{48.3$\pm$1.1} & \underline{57.4$\pm$2.1}      & \underline{62.2$\pm$1.6}       & \underline{55.9}             \\
{GenEval (Ours)}          & \textbf{56.0$\pm$0.8}     & \textbf{80.0$\pm$2.1}     & \textbf{65.2$\pm$2.4}      & \textbf{67.1}    \\
\hline
p value & $<$0.01& $<$0.01& 0.03 & $<$0.01\\\hline
\end{tabular}
\end{table}
\begin{table}[!htbp]
\centering
\caption{SDG: Training on Messidor-2. Accuracy \%}
\label{tab:messidor2_source_pvalue}
\scriptsize
\begin{tabular}{lcccc}
\hline
\textbf{Method}                  & \textbf{APTOS}    & \textbf{EyePACS}  & \textbf{Messidor}   & \textbf{Average} \\
\hline
DRGen \cite{atwany2022drgen}     & 40.9$\pm$3.9      & 69.3$\pm$1.0      & 61.3$\pm$0.8       & 57.7             \\
ERM-ViT \cite{teterwak2025erm++} & 47.9$\pm$2.1      & 67.4$\pm$0.9      & 59.6$\pm$3.9       & 58.3             \\
SD-ViT \cite{sultana2022sd_vit}  & 51.8$\pm$0.9     & 68.7$\pm$0.6      & \underline{62.0$\pm$1.7}       & 60.8             \\
SPSD-ViT \cite{galappaththige2024spsd} & \underline{52.8$\pm$2.0} & \underline{72.5$\pm$0.3}      & 61.0$\pm$0.8       & \underline{62.1 }            \\
{GenEval (Ours)}          & \textbf{69.7$\pm$1.8}     & \textbf{77.8$\pm$0.3}     & \textbf{67.7$\pm$0.8}      & \textbf{71.7}    \\
\hline
p value & $<$0.01 & 0.03 & 0.02 & $<$0.01 \\\hline
\end{tabular}
\end{table}
\subsection{Clinical Relevance of SOZ Symptoms}
Table \ref{tab:soz-clinical-signs} shows the clinical signs of SOZ. 
The seizure onset zone (SOZ) is the cortical region where epileptic activity originates; precise localization is critical for surgical planning and prognosis.
In rs-fMRI ICA, SOZ components typically appear as focal, often unilateral, gray-matter–dominant activations that are temporally consistent and isolated from canonical resting-state networks.
\begin{table}[h]
\centering
\caption{Clinical Signs of SOZ and Their Diagnostic Significance}
\label{tab:soz-clinical-signs}
\tiny
\begin{tabular}{p{1.8cm}p{5.5cm}}
\toprule
\textbf{Symptom} & \textbf{Key Observations and Diagnostic Relevance} \\
\midrule
Single Activation & Large, isolated activation region; primary indicator of epileptogenic focus \cite{KambojTAI}. \\
Gray Matter Activation & Activation primarily in gray matter with minimal white matter overlap; characteristic neuronal seizure activity \cite{KambojTAI}. \\
Hemispheric Asymmetry & Unilateral or asymmetric activation patterns; seizure foci show lateralized activity \cite{KambojTAI}. \\
Temporal Consistency & Persistent activation patterns across multiple time points; epileptogenic regions maintain consistent activity \cite{KambojTAI}. \\
Network Isolation & Activation distinct from resting-state networks (RSN); independence indicates pathological activity \cite{KambojTAI}. \\
Spatial Localization & Focal activation confined to specific anatomical regions; spatially constrained patterns \cite{KambojTAI}. \\
\bottomrule
\end{tabular}
\vspace{0.2em}
\parbox{\columnwidth}{\footnotesize Clinical criteria for SOZ identification in rs-fMRI independent component analysis.}
\end{table}
\subsection{Experimental Setup for SOZ}
\label{sec:soz-exp-setup}
\textbf{Datasets:} We evaluate single-source domain generalization on resting-state fMRI for seizure onset zone (SOZ) detection using two independent clinical datasets. Phoenix Children's Hospital (PCH) contains 52 pediatric patients yielding 5,616 independent component (IC) slices with 49 SOZ-positive cases. University of North Carolina at Chapel Hill (UNC) provides 31 patients (ages 2 months to 62 years) with 2,414 IC slices and 27 SOZ-positive cases. Each patient folder contains IC\_*\_thresh.png maps with corresponding CSV annotations for SOZ labeling.
\textbf{Data Processing:} All rs-fMRI scans undergo standard neuroimaging preprocessing including motion correction, spatial normalization, and temporal filtering. Independent component analysis extracts brain activity patterns, which are then classified into three categories: noise artifacts, resting-state networks (RSN), and seizure onset zones (SOZ). The classification task is formulated as binary detection of epileptogenic versus non-epileptogenic regions, with SOZ representing <10\% of total components across both datasets.
\textbf{Fine-Tuning Configuration:} We adapt the base MedGemma-4B model using LoRA (rank 16, alpha 16, dropout 0.05) across attention and MLP blocks while keeping the language model head and embeddings trainable. Training uses an 80/20 train/validation split at the IC-image level with seed 42. Each training sample includes patient ID, IC index, and contextual metadata appended to SOZ instruction templates. Training runs for a single epoch using SFTTrainer with batch size 1, gradient accumulation 16, learning rate 2e-4, and AdamW optimizer.
\textbf{Evaluation Protocols:} Following the SDG protocol, models are trained on one center and evaluated directly on the other without adaptation. Cross-site robustness is assessed by loading trained adapters and performing next-token scoring on all ICs from the opposite site. Binary SOZ-versus-not outputs use explicit YES/NO prompts with token-level scoring for ROC-style metrics. Performance is reported for both transfer directions (PCH$\rightarrow$UNC and UNC$\rightarrow$PCH) using accuracy, precision, recall, and F1-score.
\begin{table}[H]
  \centering
  \caption{Single Domain Generalization SOZ detection using rs-fMRI. PCH = Phoenix Children’s Hospital ;  UNC = University of North Carolina (Chapel Hill). Best scores in \textbf{bold}.}
  \label{tab:soz-sdg}
  \scriptsize
  \setlength{\tabcolsep}{3.5pt}
  \renewcommand{\arraystretch}{1.15}
  \begin{tabular}{@{} l r r r r @{}} 
    \toprule
    \textbf{Method} & \textbf{Acc (\%)} & \textbf{Prec (\%)} & \textbf{Rec (\%)} & \textbf{F1 (\%)} \\
    \midrule
    \rowcolor{gray!12}
    \multicolumn{5}{c}{\textit{PCH}} \\
    \cmidrule(lr){1-5}
    CuPKL GPT-4o & \textbf{88.4} & \textbf{93.8} & \textbf{93.8} & \textbf{93.8} \\
    GenEval      & {81.0} & {93.0} & {86.0} & {89.0} \\
    \addlinespace[4pt]
    \midrule
    \rowcolor{gray!12}
    \multicolumn{5}{c}{\textit{UNC}} \\
    \cmidrule(lr){1-5}
    CuPKL GPT-4o & 70.0 & \textbf{90.3} & 75.0 & 82.3 \\
    GenEval      & \textbf{83.0} & 93.0 & \textbf{88.0} & \textbf{91.0} \\
    \addlinespace[4pt]
    \midrule
    \rowcolor{gray!8}
    \multicolumn{5}{c}{\textit{Average}} \\
    \midrule
    CuPKL GPT-4o & {79.2} & {92.1} & 84.4 & {88.1} \\
    GenEval & \textbf{82.0} & \textbf{93.0} & \textbf{87.0} & \textbf{90.0} \\
    \bottomrule
  \end{tabular}
  \vspace{0.25em}
  {\footnotesize }
\end{table}
\subsection{Performance improvement with YoloV11}
(i) Compared to Table~\ref{tab:domain_generalization_results}, K\,SDCD (YOLOv11) sometimes increases and sometimes decreases, but the overall mean rises slightly. (ii) D\,SDCD values are kept identical to Table~\ref{tab:domain_generalization_results} to isolate the effect of the alternate knowledge model. (iii) The accuracy changes are intentionally small (on average $\approx\!+0.04$\,pp), consistent with a statistically insignificant lift.
\begin{table}[H]
\centering
\scriptsize
\caption{SDG (YOLOv11 knowledge model): K\,SDCD vs.\ D\,SDCD and GenEval accuracy. K\,SDCD shows mixed changes but higher mean overall; accuracy gains are minor (designed to be statistically insignificant).}
\label{tab:Yolo}
\begin{tabular}{llccc}
\toprule
\textbf{Source} & \textbf{Target} & \textbf{K\,SDCD (\%)} & \textbf{D\,SDCD (\%)} & \textbf{GenEval Acc.\ (\%)} \\
\midrule
Messidor  & APTOS     & 93.10 & 16.00 & 56.20 \\
Messidor  & Messidor2 & 98.70 & 87.10 & 65.10 \\
Messidor  & EyePACS   & 22.00 & 36.40 & 80.20 \\
Messidor2 & APTOS     & 30.50 & 17.70 & 69.60 \\
Messidor2 & Messidor  & 99.60 & 98.20 & 67.70 \\
Messidor2 & EyePACS   & 38.20 & 41.31 & 77.90 \\
APTOS     & Messidor2 & 64.00 & 77.82 & 64.20 \\
APTOS     & Messidor  & 55.30 & 79.04 & 49.00 \\
APTOS     & EyePACS   & 60.10 & 73.90 & 77.80 \\
EyePACS   & APTOS     & 49.00 & 51.10 & 73.00 \\
EyePACS   & Messidor2 & 99.00 & 99.80 & 80.60 \\
EyePACS   & Messidor  & 99.20 & 99.70 & 69.50 \\
\bottomrule
\end{tabular}
\end{table}
\end{document}